\documentclass[conference]{IEEEtran}
\IEEEoverridecommandlockouts
\usepackage{cite}
\usepackage{amsmath,amssymb,amsfonts}
\usepackage{mathptmx}
\usepackage{amsthm}
\usepackage{algorithmic}
\usepackage{algorithm}
\usepackage{graphicx}
\usepackage{textcomp}
\usepackage{xcolor}
\usepackage{booktabs}
\IfFileExists{multirow.sty}{\usepackage{multirow}}{}
\providecommand{\multirow}[3]{#3}
\usepackage{url}
\usepackage{bm}
\usepackage{tikz}
\usetikzlibrary{arrows.meta,positioning,fit,backgrounds,calc,shapes.geometric,
                shapes.symbols,decorations.pathmorphing,decorations.pathreplacing,
                patterns}
\IfFileExists{fontawesome5.sty}{\usepackage{fontawesome5}}{}

\definecolor{figink}{HTML}{1A1A1A}   
\definecolor{figmute}{HTML}{8A8A8A}  
\definecolor{figrule}{HTML}{D4D4D4}  
\definecolor{figfill}{HTML}{F4F4F2}  
\definecolor{figacc}{HTML}{1B6B57}   
\definecolor{figaccbg}{HTML}{E6EFEB}
\definecolor{figalt}{HTML}{A8322B}   
\definecolor{figaltbg}{HTML}{F6E9E7}

\tikzset{
  figbase/.style={
    font=\footnotesize,
    >={Stealth[length=4.6pt,width=3.4pt]},
    every node/.append style={text=figink},
  },
  fighead/.style={font=\scriptsize\bfseries, text=figink},
  figlabel/.style={font=\scriptsize, text=figink},
  fignote/.style={font=\scriptsize, text=figmute},
  figbox/.style={rounded corners=1.5pt, draw=figrule, line width=0.5pt,
                 fill=figfill, align=center, inner sep=5pt},
  figplain/.style={figbox, draw=none},
  figaccbox/.style={figbox, fill=figaccbg, draw=figacc!35},
  figaltbox/.style={figbox, fill=figaltbg, draw=figalt!30},
  figflow/.style={->, draw=figmute, line width=0.6pt},
  figthin/.style={->, draw=figrule!60!figmute, line width=0.45pt},
  figedge/.style={draw=figrule, line width=0.5pt},
}

\let\origincludegraphics\includegraphics
\newsavebox{\figframebox}
\newcommand{\figframewidth}{0.4pt}
\renewcommand{\includegraphics}[2][]{%
  \sbox\figframebox{\origincludegraphics[#1]{#2}}%
  \begin{tikzpicture}[baseline=(figbb.south)]
    \node[inner sep=0pt,outer sep=0pt] (figbb) {\usebox\figframebox};
    \useasboundingbox (figbb.south west) rectangle (figbb.north east);
    \draw[figrule,line width=\figframewidth]
      (figbb.south west) rectangle (figbb.north east);
  \end{tikzpicture}}

\newtheorem{theorem}{Theorem}

\newtheorem{proposition}{Proposition}
\newtheorem{corollary}{Corollary}

\newtheorem{assumption}{Assumption}
\newtheorem{remark}{Remark}

\def\BibTeX{{\rm B\kern-.05em{\sc i\kern-.025em b}\kern-.08em
    T\kern-.1667em\lower.7ex\hbox{E}\kern-.125emX}}

\def\IEEEbibitemsep{0pt plus 0.3pt}
\makeatletter
\IEEEtriggercmd{\reset@font\scriptsize}
\makeatother
\IEEEtriggeratref{1}

\makeatletter
\def\@IEEEtabletopskip{0pt}
\renewcommand{\tablename}{TABLE}
\makeatother
\usepackage{caption}
\newcommand{\snorm}{s_{\mathrm{norm}}}
\newcommand{\sbase}{s_{\mathrm{base}}}
\newcommand{\pmax}{p_{\max}}

\newcommand{\covstep}{\mathrm{Cov}_{\mathrm{step}}}
\newcommand{\covsim}{\mathrm{Cov}_{\mathrm{traj}}}
\newcommand{\Cset}{C_\alpha}
\newcommand{\Cdep}{\bar C_\alpha}

\begin{document}
\title{ENCP: Episode-Normalized Conformal Prediction for Vision-and-Language Navigation}

\author{
    \IEEEauthorblockN{
        Vicky Feliren$^{1,2}$, 
        A. Taufiq Asyhari$^{1}$, 
        Muhamad Risqi U. Saputra$^{1}$
    }
    \IEEEauthorblockA{
        $^{1}$Monash University, Indonesia \qquad $^{2}$SEACrowd
    }
}
\maketitle

\begin{abstract}
Uncertainty estimation for Vision-Language-Navigation (VLN) models is a critical task since it can help identify ambiguous and unreliable predictions, enabling agents to make safer navigation decisions. As one of the most advanced uncertainty estimation frameworks, conformal prediction (CP) offers a promising approach for uncertainty estimation in VLN. However, given that VLN agent requires a sequence of steps, standard calibration in conformal prediction fails to provide coverage guarantee it promises over a dependent, variable-length VLN episode. To this end, we propose Episode-Normalized Conformal Prediction (ENCP), which rescales a nonconformity score by the policy's residual confidence and calibrates one maximum score per episode. Under exchangeable calibration and test episodes, this construction covers the ground truth at every step with probability at least $1-\alpha$, while allowing dependence among steps within an episode. Across four VLN policies and three nonconformity scores on R2R and REVERIE dataset, ENCP meets all reported empirical step-coverage targets on the seen-to-unseen evaluation. These results demonstrate that ENCP can provide model-agnostic uncertainty estimates, which might be useful for determining when a VLN agent should defer to a more capable predictor, including human assistance.

\end{abstract}


\begin{figure*}[t]
\centering
\includegraphics[width=\textwidth]{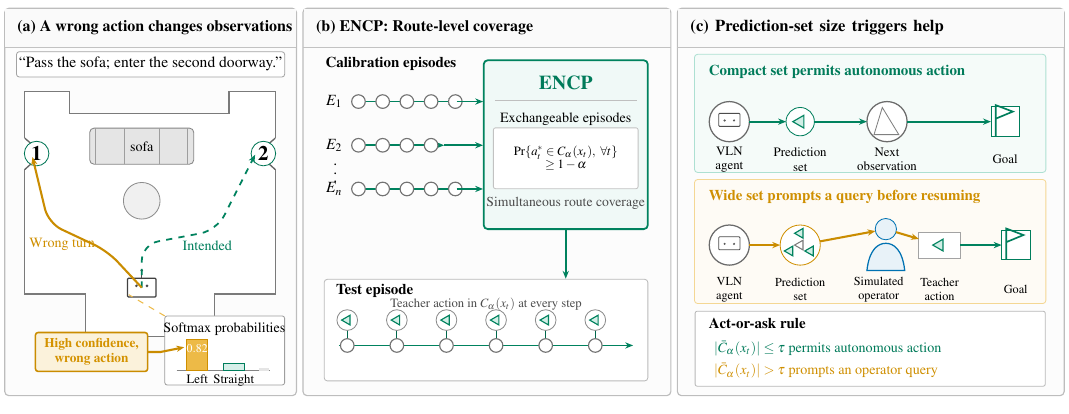}
\caption{ENCP identifies when a navigation agent should ask for help before an
error compounds. It calibrates complete episodes and triggers a query when the
action set exceeds a chosen size budget.}
\label{fig:teaser}
\end{figure*}

\section{Introduction}
\label{sec:intro}

Vision-and-Language Navigation (VLN) enables embodied agents to navigate physical environments using natural language instructions~\cite{anderson2018r2r,zhang2024survey}. Such capabilities have applications in assistive robotics, household robots, and autonomous systems operating in complex environments where agents must follow human instructions while adapting to their surroundings. Deploying these agents safely in real-world settings, however, demands reliable uncertainty estimation as sequential navigation inherently suffers from compounding errors: a single wrong turn corrupts subsequent visual feedback, forcing the agent to act on invalid environmental states~\cite{ross2011reduction}. While an uncorrected mistake merely degrades a benchmark score, continuing confidently while lost in physical spaces risks collisions or task failure. Consequently, agents must identify unreliable steps and request human intervention before errors accumulate~\cite{amodei2016concrete,hendrycks2021unsolved,lindemann2023conformal}. This might requires simultaneously resolving linguistic ambiguity and visual uncertainty across dynamic trajectories (e.g., grounding ``enter the second door'').

Modern VLN policies achieve high navigation success~\cite{hong2021vlnbert,chen2021hamt,chen2022duet}, yet overall task completion metrics fail to indicate whether a specific action is trustworthy. Existing self-monitoring and help-request mechanisms recognize this limitation~\cite{ma2019selfmonitoring,abraham2025ask}, but they rely on empirical heuristics rather than formal calibration tied to a user-specified error rate.

The policy's largest softmax probability is often used to flag how certain the decisions, with smaller values indicating that the policy has no strongly preferred action; however, this value is not a calibrated probability that the selected action is incorrect. Because neural networks can be
overconfident~\cite{guo2017calibration} and uncertainty estimates often
deteriorate when the deployment environment differs from the training data
\cite{ovadia2019trust}, a cutoff has no distribution-free
interpretation without a calibration argument. This matters in human--automation
teams, where a useful warning should support appropriate reliance rather than
encourage blind trust or constant intervention~\cite{parasuraman1997humans,
lee2004trust}.

Conformal Prediction (CP) is a statistical framework that converts point predictions into valid prediction sets. It returns a set of plausible actions at a user-chosen
miscoverage (risk) level $\alpha$, whose target coverage is $1-\alpha$. A one-action set
lets the agent proceed. A larger set tells it to ask. When the calibration
routes and new routes are generated under the same conditions, CP guarantees
that the set includes the correct choice at the requested rate. The policy's
probabilities need not be calibrated
\cite{vovk2005alrw,lei2018distributionfree,angelopoulos2023gentle}. This
act-or-ask interpretation has already made CP useful for robot planning and
human assistance~\cite{ren2023knowno,liang2024introspective}.


In the standard CP, the coverage guarantee usually applies to a single prediction. However, this approach will not work for VLN since it does not ensure that the action deemed correct is included at every step along a route~\cite{vovk2005alrw,angelopoulos2023gentle}. Each action changes the agent's state and
therefore affects the observations and decisions that follow~\cite{ross2011reduction}.
Using individual steps as calibration examples overlooks this dependence and provides no guarantee for the route as a whole. A further limitation arises when the policy assigns nearly all of its probability to one action. This motivates the need to design a bespoke CP that can ensure its coverage guarantee over a dependent, variable-length VLN episode.  

In this work, we propose Episode-Normalized Conformal Prediction (ENCP), a post-training conformal prediction method for VLN that leaves the underlying navigation policy unchanged. ENCP rescales each step-level nonconformity score by the policy’s residual confidence and calibrates one maximum normalized score per episode. Under exchangeability of the calibration and test episodes, the resulting prediction sets contain the ground-truth action at every step of a test episode with probability at least $1-\alpha$, while allowing dependence among steps within an episode and accommodating variable episode lengths. Episode-level calibration provides this coverage guarantee, whereas residual-confidence rescaling helps calibrated thresholds transfer across the evaluated policies.

This work makes 3 contributions.
\begin{itemize}
    \item We show how standard, step-wise CP can fail to satisfy its coverage guarantee in VLN and measure when its action
    sets cease to reflect the miscoverage level $\alpha$.
    \item We develop ENCP, an episode-level calibration CP with a guarantee that covers the correct action throughout a complete route.
    \item We evaluate ENCP against step-pooled CP using 4 VLN models and 3 nonconformity scores on the R2R and REVERIE datasets, examine coverage when calibration and test buildings differ, and use prediction-set size to decide when to request an action from a simulated ground-truth assistant.
\end{itemize}

\section{Related Work}
\label{sec:related}

\textbf{Vision-and-language navigation.} Room-to-Room (R2R) introduced instruction following on the Matterport3D navigation graph~\cite{anderson2018r2r,chang2017matterport}, and VLN later expanded to continuous control~\cite{krantz2020vlnce} and remote object grounding in REVERIE (Remote Embodied Visual Referring Expression in Real Indoor Environments)~\cite{qi2020reverie}. Recurrent sequence models~\cite{tan2019envdrop} preceded transformer policies such as DUET, HAMT, and Recurrent VLN-BERT \cite{chen2022duet,chen2021hamt,hong2021vlnbert}, evaluated here. Learned help-request heads detect ambiguous instructions \cite{abraham2025ask}; however, their reported reliability is empirical, failing to provide a finite-sample guarantee at a risk level chosen by the user.

\textbf{Uncertainty estimation.} Bayesian approximations \cite{gal2016dropout}, deep ensembles~\cite{lakshminarayanan2017ensembles}, and selective prediction~\cite{geifman2017selective} yield uncertainty or abstention scores, while classical robotics represents and propagates probabilistic states~\cite{thrun2005probabilistic}. These methods do not generally provide distribution-free finite-sample coverage at a chosen $\alpha$.

\textbf{Conformal prediction for decisions and sequences.} Classification CP includes threshold scores (THR), adaptive prediction sets (APS), and rank-regularized adaptive prediction sets (RAPS) \cite{sadinle2019least,romano2020aps,angelopoulos2021raps}. Beyond classification, CP covers LLM planning, model predictive control, and multi-agent motion planning \cite{liang2024introspective,lindemann2023conformal,dixit2023multiagent}. Specifically, KnowNo uses prediction-set size to request human help, calibrating weakest confidence over a fixed plan~\cite{ren2023knowno}. In contrast, ENCP calibrates maximum nonconformity over a episode-length VLN trajectory, whose action set can change at every step.

\textbf{Dependence and distribution shift.} Standard split CP assumes exchangeable calibration and test units. Methods for nonexchangeable data can bound the resulting coverage loss~\cite{barber2023beyond}, while conformal time-series methods allocate risk across a fixed horizon \cite{stankeviciute2021cfrnn}. ENCP uses the complete episode as its calibration unit, so steps within an episode may depend on one another and no fixed horizon is needed. However, this design choice does not remove the distribution shift. R2R calibrates on seen buildings and evaluates on unseen buildings. Weighted CP can correct covariate shift when density ratios are available \cite{tibshirani2019covariateshift}. We report coverage under exchangeable splits and across the seen-to-unseen boundary.

\section{Problem Formulation}
\label{sec:problem}
In this section, we formulate VLN as sequential classification in discrete graph, define the 3 nonconformity
scores used in commonly used CP, and show why calibration over pooled steps does not yield trajectory-level coverage. An experimental condition is one policy--dataset pair.

\subsection{VLN as a discrete graph}
An \emph{episode} is one attempt to follow an instruction $L$ from a given
start node. The agent moves on an undirected graph of panoramic Matterport3D
viewpoints~\cite{chang2017matterport} and terminates by selecting
\textsc{stop}. At step $t$, let $o_t$ denote the current panorama, $h_t$ the
preceding observation--action history, and $\mathcal A_t$ the admissible
actions, comprising adjacent viewpoints and \textsc{stop}. We write
$x_t=(L,o_t,h_t,\mathcal A_t)$ for the complete decision context. Define the
policy probability of action $a$ by
\begin{equation}
p(a\mid x_t)=\pi_\theta(a\mid x_t),\qquad a\in\mathcal A_t,
\label{eq:policy}
\end{equation}
and define
$\pmax(x_t)=\max_{a\in\mathcal A_t}p(a\mid x_t)$. We write $\pmax$ for $\pmax(x_t)$ when considering a step $t$. The action-set size varies with the local graph degree. For example, DUET's global branch additionally considers the reachable frontier ~\cite{chen2022duet}. Section~\ref{sec:guarantee} treats each resulting episode as one calibration or test unit.

\subsection{Nonconformity scores and split-conformal calibration}
In conformal prediction, a nonconformity score measures how atypical a candidate output appears relative to a model's prediction, with higher values indicating greater disagreement or lower confidence.
In VLN, we instantiate this score as $s(x_t,a)$ and assigns larger values to actions that conform
less closely to the policy output at $x_t$. The score must be fixed before conformal calibration. For a labeled
decision context $(x_t,a_t^*)$, its calibration score, the nonconformity score assigned to the ground-truth action, is $s(x_t,a_t^*)$.

In this work, we use THR and the deterministic ($U=0$) forms of APS and RAPS as standard non-conformity scores. Here, $U \sim \operatorname{Uniform}(0,1)$ is an auxiliary variable that randomizes the contribution of the candidate action's probability mass; setting $U=0$ removes this randomization. THR is
the complement of the candidate probability~\cite{sadinle2019least}:
\begin{equation}
s_{\mathrm{THR}}(x_t,a)=1-p(a\mid x_t).
\label{eq:thr}
\end{equation}
APS orders actions by decreasing $p(a\mid x_t)$, with ties resolved by a fixed rule.
Let $\operatorname{rank}_t(a)\in\{1,\ldots,|\mathcal A_t|\}$ denote the
resulting position of action $a$. APS assigns to $a$ the probability mass of
all preceding actions~\cite{romano2020aps}:
\begin{equation}
s_{\mathrm{APS}}(x_t,a)
=\sum_{b\in\mathcal A_t:\,\operatorname{rank}_t(b)<\operatorname{rank}_t(a)}
p(b\mid x_t).
\label{eq:aps}
\end{equation}
Thus the top-ranked action has score 0. RAPS adds a penalty beyond rank
$k_{\mathrm{reg}}$~\cite{angelopoulos2021raps}:
\begin{equation}
s_{\mathrm{RAPS}}(x_t,a)=s_{\mathrm{APS}}(x_t,a)
+\lambda\big[\operatorname{rank}_t(a)-k_{\mathrm{reg}}\big]_+,
\label{eq:raps}
\end{equation}
where $[u]_+=\max(u,0)$; we use $\lambda=0.1$ and $k_{\mathrm{reg}}=2$.

For standard split CP, let $s_1,\ldots,s_n$ be exchangeable calibration scores
for one fixed score function. Exchangeability in conformal prediction means that the joint probability distribution of a sequence of data points does not change when their order is permuted. Given a miscoverage level $\alpha\in(0,1)$, set
$k=\lceil(n+1)(1-\alpha)\rceil$. If $s_{(k)}$ denotes the $k$-th order
statistic, the conformal threshold is
\begin{equation}
\hat q(\alpha)=
\begin{cases}
s_{(k)}, & k\leq n,\\
+\infty, & k=n+1.
\end{cases}
\label{eq:qhat}
\end{equation}
The corresponding prediction set is
\begin{equation}
C_\alpha(x_t)=\{a\in\mathcal A_t:s(x_t,a)\leq\hat q(\alpha)\}.
\label{eq:base-set}
\end{equation}
If the calibration scores and the test score $s(x_t,a_t^*)$ are exchangeable,
then $\mathbb P\{a_t^*\in C_\alpha(x_t)\}\geq1-\alpha$.

\subsection{Why split CP degenerates here}
\label{ssec:failure}
The pooled-step baseline CP treats each labeled step as an independent calibration unit. By \eqref{eq:qhat}, the resulting threshold does not exceed $\epsilon$ whenever at least $k$ calibration scores are $\le \epsilon$. 

For $\alpha = 0.30$, APS and RAPS return only the policy’s maximum-probability action on DUET and HAMT, with THR exhibiting analogous behavior. Split CP similarly demonstrates undercoverage across the evaluated configurations (Figure~\ref{fig:base-failure}). Moreover, modifying the policy’s probability calibration does not resolve this quantile-selection phenomenon (Section~\ref{sec:disc}).

\begin{figure}[t]
\centering
\includegraphics{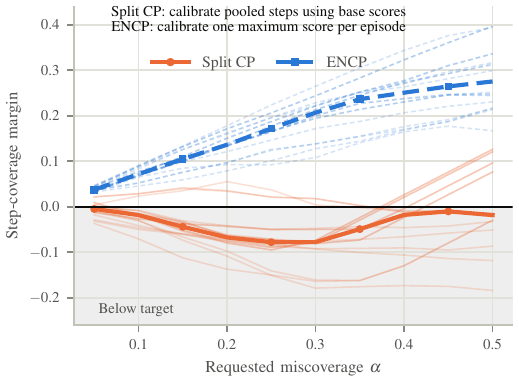}
\caption{Split CP undercovers across most settings, making its nominal risk
level unreliable here. ENCP calibrates one maximum normalized score per episode
and meets the target throughout. Negative values indicate undercoverage.}
\label{fig:base-failure}
\end{figure}

Split CP assumes that calibration and test scores are exchangeable. However, exchangeability at the level of complete episodes does not generally entail exchangeability of the individual time steps obtained by pooling steps across episodes. Within an episode, time steps are statistically dependent due to their shared action–observation history; episode lengths may differ; and actions can causally influence subsequent observations, further violating step-wise exchangeability. Moreover, marginal coverage guarantees for a single time step do not imply simultaneous coverage across all time steps within an episode. Accordingly, the subsequent section defines a mapping from each episode to a single calibration score.

\section{Episode-Normalized Conformal Prediction}
\label{sec:method}
Standard CP calibration fails to provide coverage guarantees in sequential VLN due to intra-episode step dependencies and dynamic action spaces. To resolve this, Episode-Normalized Conformal Prediction (ENCP) derives its name from two core mechanisms: \textit{episode-level calibration}, which collapses dependent trajectory steps into a single worst-case metric to restore finite-sample guarantees, and \textit{score normalization}, which adjusts nonconformity scores relative to step-varying policy confidence. ENCP can be implemented as either a parameter-free variant using raw policy probabilities or a learning-based variant that fits score weights prior to calibration (Figure~\ref{fig:pipeline}).

\begin{figure*}[t]
\centering
\includegraphics[width=\textwidth]{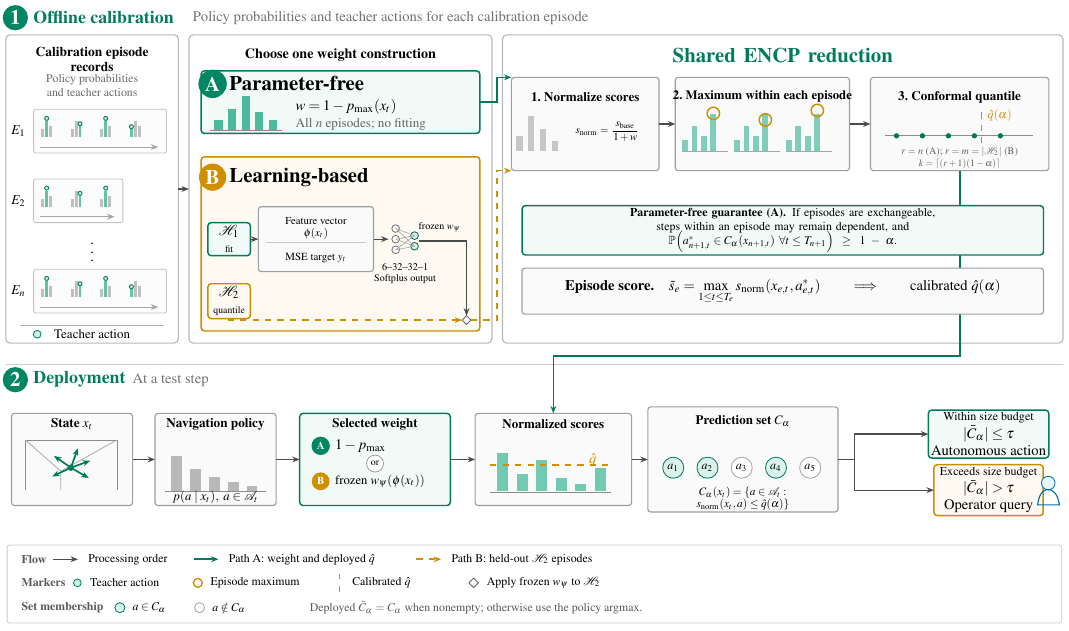}
\caption{ENCP calibrates one worst-step score per episode, then asks for help
when the deployed set exceeds $\tau$.}
\label{fig:pipeline}
\end{figure*}

\subsection{Confidence-adjusted score}
\label{ssec:score}
For any base score $\sbase\in\{s_{\mathrm{THR}},s_{\mathrm{APS}},
s_{\mathrm{RAPS}}\}$ and fixed nonnegative weight rule $w$, ENCP defines
\begin{equation}
\snorm(x_t,a)=\frac{\sbase(x_t,a)}{1+w(x_t)}.
\label{eq:snorm}
\end{equation}

Here, $x_t$ denotes the decision context at step $t$, $a\in\mathcal A_t$ is a candidate action, $\sbase$ is the nonconformity score, $w(x_t)$ is its confidence-dependent weight, and $\snorm$ is the resulting confidence-adjusted nonconformity score.

\paragraph{Parameter-free weight}
\begin{equation}
\begin{aligned}
w_{\mathrm{pf}}(x_t) &= 1-\pmax(x_t),\\
\snorm^{\mathrm{pf}}(x_t,a)
&=\frac{\sbase(x_t,a)}{2-\pmax(x_t)}.
\end{aligned}
\label{eq:pf-weight}
\end{equation}
At a fixed step, every candidate shares the same positive denominator. The
normalization therefore preserves the policy's action ranking. It changes set
membership through the effective threshold
\begin{equation}
a\in\Cset(x_t)
\quad\Longleftrightarrow\quad
\sbase(x_t,a)\leq\hat q(\alpha)[2-\pmax(x_t)].
\label{eq:pf-membership}
\end{equation}
Thus, lower confidence increases the effective threshold and can admit more
actions. Episode
calibration addresses this remaining point mass by changing the calibration
unit.

\paragraph{Learning-based weight}
The learning-based variant uses step features to predict the weight. We define
\begin{equation}
\begin{aligned}
\phi(x_t)&=\left[H(p_t),\ \pmax,\ p_{(1)}-p_{(2)},\right.\\
&\hspace{2.2em}\left.\log|\mathcal A_t|,\ t/T_{\max},\ \alpha\right],\\
w_\psi(x_t)&=\operatorname{softplus}\!\left(f_\psi(\phi(x_t))\right)\geq0,
\end{aligned}
\label{eq:weight-features}
\end{equation}
Here, $H(p_t)$ is entropy over the candidate actions, $\pmax$ is the largest
action probability, and $p_{(1)}-p_{(2)}$ is the top-two probability gap.
The term $\log|\mathcal A_t|$ is the log action-set size, $t/T_{\max}$ is the
normalized step index, and $\alpha$ is the requested miscoverage level. The
network $f_\psi$ has two ReLU hidden layers of width $32$; its Softplus output
keeps the denominator in \eqref{eq:snorm} positive.

We train the network to assign more weight to steps where the policy is
confident but gives little probability to the ground-truth action. For each labeled
step in fit subset $\mathcal H_1$, the target is
\begin{equation}
y_t=\operatorname{clip}\!\left(
\frac{1-p(a_t^*\mid x_t)}{1-\pmax(x_t)},0,10\right).
\label{eq:weight-target}
\end{equation}
Clipping limits values caused by a denominator near zero. We then minimize
\begin{equation}
\mathcal L(\psi)=\frac{1}{N_1}\sum_{(e,t)\in\mathcal H_1}
\left[w_\psi(x_{e,t})-y_{e,t}\right]^2.
\label{eq:weight-loss}
\end{equation}
After fitting on $\mathcal H_1$, we freeze $w_\psi$ and estimate the conformal
quantile on disjoint subset $\mathcal H_2$. Conditional on $\mathcal H_1$, the
score is fixed before calibration, so $\mathcal H_2$ remains exchangeable with
a test episode. Using $\mathcal H_2$ for both fitting and calibration would
invalidate this split-conformal argument.

\subsection{One calibration score per episode}
\label{ssec:epmax}
For calibration episode $e$ of length $T_e$, ENCP computes
\begin{equation}
\tilde s_e=\max_{1\leq t\leq T_e}
\snorm(x_{e,t},a^*_{e,t}).
\label{eq:epmax}
\end{equation}
Every episode contributes one number, regardless of its length. We require
exchangeability across episodes but make no such assumption about steps within
an episode. For parameter-free ENCP, bounding the maximum bounds every step and
gives simultaneous trajectory coverage (Theorem~\ref{thm:cov}). The maximum is zero only when all
step scores are zero. This condition makes a zero threshold less common than
under step-pooled calibration (Proposition~\ref{prop:atom}).

\subsection{Calibration and deployment}
\label{ssec:infer}
Let $r=n$ for the parameter-free variant and
$r=m=|\mathcal H_2|$ for the learning-based variant, as in
Figure~\ref{fig:pipeline}. Relabel the corresponding episode scores as
$\tilde s_1,\ldots,\tilde s_r$, and set
$k=\lceil(r+1)(1-\alpha)\rceil$. ENCP uses
\begin{equation}
\hat q(\alpha)=
\begin{cases}
\tilde s_{(k)}, & k\leq r,\\
+\infty, & k=r+1,
\end{cases}
\label{eq:qhat-ep}
\end{equation}
where $\tilde s_{(k)}$ is the $k$-th order statistic. At a test step,
thresholding produces the raw prediction set
\begin{equation}
\Cset(x_t)=\{a\in\mathcal A_t:
\snorm(x_t,a)\leq\hat q(\alpha)\}.
\label{eq:set}
\end{equation}
The deployed set is
\begin{equation}
\Cdep(x_t)=
\begin{cases}
\Cset(x_t), & \Cset(x_t)\neq\varnothing,\\
\{\arg\max_{a\in\mathcal A_t}p(a\mid x_t)\}, & \text{otherwise},
\end{cases}
\label{eq:set-fallback}
\end{equation}
where the fixed policy tie-breaking rule resolves the argmax. For help-seeking,
an operator may choose a size budget $\tau$ and
request assistance whenever $|\Cdep(x_t)|>\tau$
(Section~\ref{ssec:closedloop}).

\section{Coverage Guarantee}
\label{sec:guarantee}
This section proves coverage for parameter-free ENCP. For this purpose, before calibration, we freeze the navigation policy, a base score $\sbase$, and
$w_{\mathrm{pf}}(x_t)=1-\pmax(x_t)$. For episode $E_i$,
let $x_{i,t}$ and $a_{i,t}^*$ denote the decision context and ground-truth action at
step $t$, and let $T_i=T(E_i)\geq1$. Throughout this section, $\snorm$ denotes
the parameter-free score. The fixed episode-score map is
\begin{equation}
\varphi(E_i)=\tilde s_i
=\max_{1\leq t\leq T_i}
\frac{\sbase(x_{i,t},a_{i,t}^*)}{2-\pmax(x_{i,t})}.
\label{eq:pf-episode-score}
\end{equation}

\begin{assumption}[Exchangeable episodes]
\label{asm:exch}
The joint distribution of the $n$ calibration episodes and one test episode,
$(E_1,\ldots,E_n,E_{n+1})$, is invariant under permutations of their indices.
In particular, the assumption holds for i.i.d. episodes from one task
distribution.
\end{assumption}

This assumption applies to complete episodes. Steps within an episode may
depend on one another, episode lengths may vary, the navigation policy may be
miscalibrated, and the scores may contain ties or point masses.

For the test episode, let
\begin{equation}
\mathcal E_{n+1}
=\bigcap_{t=1}^{T_{n+1}}
\{a_{n+1,t}^*\in \Cset(x_{n+1,t})\}
\label{eq:coverage-event}
\end{equation}
be the event that every ground-truth action is covered. Define simultaneous
trajectory coverage $\covsim$ and step-averaged coverage $\covstep$ by
\begin{align}
\covsim
&=\mathbb P(\mathcal E_{n+1}),
\label{eq:cov-traj}\\
\covstep
&=\mathbb E\!\left[
\frac{1}{T_{n+1}}\sum_{t=1}^{T_{n+1}}
\mathbf 1\{a_{n+1,t}^*\in \Cset(x_{n+1,t})\}
\right].
\label{eq:cov-step}
\end{align}

\begin{theorem}[Simultaneous trajectory coverage]
\label{thm:cov}
Fix $\alpha\in(0,1)$ and let
$k=\lceil(n+1)(1-\alpha)\rceil$. Define $\hat q(\alpha)$ as the $k$-th
smallest value among the calibration scores
$\tilde s_1,\ldots,\tilde s_n$, with $\hat q(\alpha)=+\infty$ if $k=n+1$.
Under Assumption~\ref{asm:exch},
\begin{equation}
\mathbb{P}(\mathcal E_{n+1})\geq 1-\alpha.
\label{eq:thm-cov}
\end{equation}
\end{theorem}

\begin{proof}
Let $\hat q=\hat q(\alpha)$. By \eqref{eq:set} and
\eqref{eq:pf-episode-score},
\begin{equation}
\begin{split}
\mathcal E_{n+1}
&=\left\{
\max_{1\leq t\leq T_{n+1}}
\snorm(x_{n+1,t},a_{n+1,t}^*)\leq\hat q
\right\}\\
&=\{\tilde s_{n+1}\leq\hat q\}.
\end{split}
\label{eq:event}
\end{equation}

Because $\varphi$ is fixed, the scores
$\tilde s_1,\ldots,\tilde s_{n+1}$ are exchangeable. Suppose first that
$k\leq n$, and let $u_{(k)}$ be their $k$-th order statistic. Removing
$\tilde s_{n+1}$ cannot decrease the $k$-th order statistic, so
$\hat q\geq u_{(k)}$. The exchangeable-rank bound
\cite{angelopoulos2023gentle} gives
\begin{equation}
\mathbb{P}(\tilde s_{n+1}\leq\hat q)
\geq \mathbb{P}(\tilde s_{n+1}\leq u_{(k)})
\geq\frac{k}{n+1}\geq1-\alpha.
\end{equation}
If $k=n+1$, then $\hat q=+\infty$ and the probability is 1. Equation
\eqref{eq:event} now proves \eqref{eq:thm-cov}. Since
$\Cset(x_t)\subseteq\Cdep(x_t)$, the deployed set has at
least the same coverage.
\end{proof}

\begin{corollary}[Step-averaged coverage]
\label{cor:step}
Under Assumption~\ref{asm:exch}, the threshold of
Theorem~\ref{thm:cov} also satisfies $\covstep\geq1-\alpha$.
\end{corollary}
\begin{proof}
Let
\[
R_{n+1}=\frac{1}{T_{n+1}}\sum_{t=1}^{T_{n+1}}
\mathbf 1\{a_{n+1,t}^*\in \Cset(x_{n+1,t})\}
\]
be the covered-step fraction in the test episode. Pointwise,
\[
R_{n+1}\geq\mathbf 1\{\mathcal E_{n+1}\}.
\]
Taking expectations and applying Theorem~\ref{thm:cov} gives
$\covstep\geq\covsim\geq1-\alpha$.
\end{proof}

Figure~\ref{fig:theorem-check} evaluates parameter-free ENCP under the theorem's
premise and across the shifted benchmark split. Across 300 random within-pool
splits, mean
$\covsim$ differs from nominal by at most $0.0017$ over the 6 reported
conditions and 3 $\alpha$ levels. Coverage is lower in the seen-to-unseen
comparison, where episode exchangeability is not assumed.

\begin{figure*}[t]
\centering
\includegraphics[width=\textwidth]{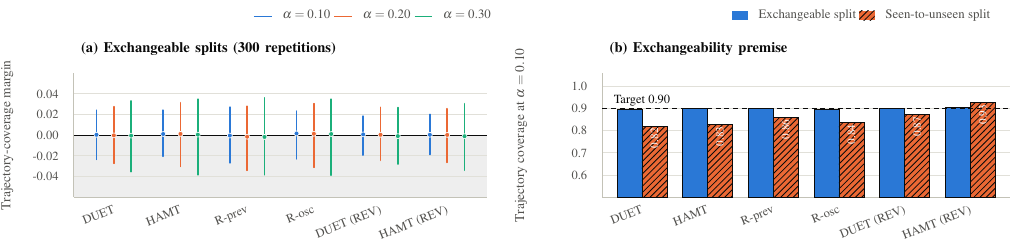}
\caption{Parameter-free ENCP attains nominal trajectory coverage under episode
exchangeability but not under seen-to-unseen shift. Error bars are empirical
95\% ranges over 300 exchangeable splits; the shifted bars lie outside
Theorem~\ref{thm:cov}'s premise.}
\label{fig:theorem-check}
\end{figure*}

The guarantee in Theorem~\ref{thm:cov} is marginal over the calibration and
test episodes; it does not condition on a realized calibration
sample~\cite{lei2018distributionfree}. Ties may make the bound conservative.
The episode maximum avoids allocating $\alpha$ across time.

\begin{proposition}[Condition for a zero threshold]
\label{prop:atom}
For parameter-free ENCP, all 3 base scores in Section~\ref{sec:problem} are
nonnegative. For any calibration episode $e$,
\begin{equation}
\tilde s_e=0
\quad\Longleftrightarrow\quad
\sbase(x_{e,t},a_{e,t}^*)=0\ \text{for every }t=1,\ldots,T_e.
\label{eq:atom}
\end{equation}
Consequently, $\hat q(\alpha)=0$ exactly when at least
$k=\lceil(n+1)(1-\alpha)\rceil$ of the $n$ calibration episodes have zero
score at every step.
\end{proposition}

\begin{proof}
The denominator $2-\pmax(x_t)$ lies in $[1,2]$, so normalization preserves
zeros. A maximum of nonnegative numbers is zero exactly when every number is
zero. The $k$-th order statistic is zero exactly when at least $k$ observations
are zero.
\end{proof}

The construction is the score-space counterpart of calibrating minimum
confidence over a plan~\cite[Prop.~2]{ren2023knowno}, extended here to
variable-length VLN episodes with changing action sets.

\begin{remark}[Scope under distribution shift]
\label{rem:premise}
The guarantee is marginal under the episode law in
Assumption~\ref{asm:exch}. Seen and unseen buildings may induce different
episode laws, and operator intervention changes the trajectory distribution.
Section~\ref{sec:exp} evaluates the seen-to-unseen setting.
\end{remark}

\section{Experiments}
\label{sec:exp}
We calibrate on the R2R validation-seen, whose buildings occur in training, and test on validation-unseen (val-unseen) episodes from disjoint building scans~\cite{anderson2018r2r}. We also experiment this in REVERIE~\cite{qi2020reverie} for a better generalization. We evaluate four
VLNs from three architecture families: DUET~\cite{chen2022duet}, HAMT
\cite{chen2021hamt}, and Recurrent VLN-BERT with PREVALENT and OSCAR
initialization (R-prev and R-osc)~\cite{hong2021vlnbert}. Our runs reproduce
each published val-unseen success rate to within one percentage point. We record
prediction sets without changing the policy, except in the closed-loop study
of Section~\ref{ssec:closedloop}. The remaining evaluations therefore retain the
original policy's success rate. Unless noted otherwise, ENCP uses the
parameter-free score in \eqref{eq:snorm}. Base-score columns use the
step-pooled calibration analyzed in Section~\ref{ssec:failure}.

The experiments report step coverage $\covstep$, the mean fraction of covered
steps per episode. Set sizes include the argmax fallback, whereas coverage uses
the set in \eqref{eq:set}; an empty raw set is uncovered.

\subsection{Restored $\alpha$-sensitivity and step coverage}
\label{ssec:main}
Table~\ref{tab:encp} compares each base score with its ENCP counterpart at
three miscoverage levels. Across both datasets, ENCP keeps $\covstep$ above
$1-\alpha$ for every reported backbone, score, and miscoverage level. The base scores for THR, APS, and RAPS undercover in every corresponding entry. Thus, episode calibration restores a
usable response to $\alpha$, but the larger ENCP sets show the cost of meeting
the step-coverage target under the seen-to-unseen shift.

\begin{table*}[t]
\centering
\caption{ENCP meets the step-coverage target across R2R and REVERIE; base CP
undercovers. $\overline{|C|}$ is mean prediction-set size. \textbf{Bold} values satisfy the target coverage guarantee ($1-\alpha$).}
\label{tab:encp}
\footnotesize
\renewcommand{\arraystretch}{1.02}
\begin{tabular*}{\textwidth}{@{\extracolsep{\fill}} ll *{4}{c} c *{4}{c} c *{4}{c} @{}}
\toprule
& & \multicolumn{4}{c}{$\alpha{=}0.10$} & & \multicolumn{4}{c}{$\alpha{=}0.20$} & & \multicolumn{4}{c}{$\alpha{=}0.30$} \\
\cmidrule(lr){3-6}\cmidrule(lr){8-11}\cmidrule(lr){13-16}
& & \multicolumn{2}{c}{Base} & \multicolumn{2}{c}{\textbf{ENCP}} & & \multicolumn{2}{c}{Base} & \multicolumn{2}{c}{\textbf{ENCP}} & & \multicolumn{2}{c}{Base} & \multicolumn{2}{c}{\textbf{ENCP}} \\
\cmidrule(lr){3-4}\cmidrule(lr){5-6}\cmidrule(lr){8-9}\cmidrule(lr){10-11}\cmidrule(lr){13-14}\cmidrule(lr){15-16}
\textbf{VLN} & \textbf{Score} & $\covstep$ & $\overline{|C|}$ & $\covstep$ & $\overline{|C|}$ & & $\covstep$ & $\overline{|C|}$ & $\covstep$ & $\overline{|C|}$ & & $\covstep$ & $\overline{|C|}$ & $\covstep$ & $\overline{|C|}$ \\
\midrule
\multicolumn{16}{l}{\textit{R2R val-unseen}} \\
\midrule
\multirow{3}{*}{DUET} & THR & 0.854 & 2.6 & \textbf{0.965} & 6.7 &  & 0.729 & 1.4 & \textbf{0.936} & 6.1 &  & 0.608 & 1.0 & \textbf{0.897} & 5.2 \\
 & APS & 0.850 & 2.7 & \textbf{0.965} & 6.7 &  & 0.723 & 1.6 & \textbf{0.934} & 6.0 &  & 0.621 & 1.0 & \textbf{0.899} & 5.2 \\
 & RAPS & 0.859 & 2.6 & \textbf{0.972} & 5.7 &  & 0.725 & 1.5 & \textbf{0.928} & 4.0 &  & 0.621 & 1.0 & \textbf{0.879} & 3.1 \\
\midrule
\multirow{3}{*}{HAMT} & THR & 0.883 & 2.4 & \textbf{0.943} & 3.8 &  & 0.726 & 1.3 & \textbf{0.888} & 3.1 &  & 0.608 & 1.0 & \textbf{0.825} & 2.3 \\
 & APS & 0.881 & 2.4 & \textbf{0.943} & 3.8 &  & 0.720 & 1.3 & \textbf{0.888} & 3.1 &  & 0.626 & 1.0 & \textbf{0.818} & 2.2 \\
 & RAPS & 0.883 & 2.7 & \textbf{0.936} & 3.6 &  & 0.721 & 1.3 & \textbf{0.875} & 2.6 &  & 0.626 & 1.0 & \textbf{0.820} & 2.0 \\
\midrule
\multirow{3}{*}{R-prev} & THR & 0.873 & 2.7 & \textbf{0.972} & 4.5 &  & 0.756 & 1.6 & \textbf{0.951} & 4.2 &  & 0.653 & 1.2 & \textbf{0.924} & 3.9 \\
 & APS & 0.874 & 2.7 & \textbf{0.972} & 4.5 &  & 0.758 & 1.7 & \textbf{0.950} & 4.2 &  & 0.650 & 1.2 & \textbf{0.924} & 3.9 \\
 & RAPS & 0.892 & 3.1 & \textbf{0.985} & 4.8 &  & 0.758 & 1.6 & \textbf{0.964} & 4.4 &  & 0.650 & 1.2 & \textbf{0.923} & 3.7 \\
\midrule
\multirow{3}{*}{R-osc} & THR & 0.886 & 3.1 & \textbf{0.971} & 4.5 &  & 0.736 & 1.7 & \textbf{0.938} & 4.1 &  & 0.624 & 1.1 & \textbf{0.907} & 3.7 \\
 & APS & 0.883 & 3.0 & \textbf{0.971} & 4.5 &  & 0.736 & 1.7 & \textbf{0.939} & 4.1 &  & 0.619 & 1.2 & \textbf{0.908} & 3.7 \\
 & RAPS & 0.891 & 3.2 & \textbf{0.988} & 4.8 &  & 0.739 & 1.7 & \textbf{0.962} & 4.4 &  & 0.620 & 1.2 & \textbf{0.928} & 3.9 \\
\midrule
\multicolumn{16}{l}{\textit{REVERIE val-unseen, navigation head}} \\
\midrule
\multirow{3}{*}{DUET} & THR & 0.871 & 5.3 & \textbf{0.975} & 8.8 &  & 0.692 & 2.7 & \textbf{0.932} & 8.1 &  & 0.521 & 1.2 & \textbf{0.898} & 7.6 \\
 & APS & 0.869 & 5.4 & \textbf{0.976} & 8.8 &  & 0.703 & 3.0 & \textbf{0.933} & 8.1 &  & 0.539 & 1.5 & \textbf{0.897} & 7.6 \\
 & RAPS & 0.830 & 3.7 & \textbf{0.964} & 7.1 &  & 0.663 & 2.2 & \textbf{0.879} & 5.0 &  & 0.536 & 1.4 & \textbf{0.802} & 4.0 \\
\midrule
\multirow{3}{*}{HAMT} & THR & \textbf{0.929} & 4.3 & \textbf{0.990} & 5.3 &  & \textbf{0.834} & 3.3 & \textbf{0.975} & 5.1 &  & \textbf{0.718} & 2.3 & \textbf{0.943} & 4.7 \\
 & APS & \textbf{0.928} & 4.3 & \textbf{0.990} & 5.3 &  & \textbf{0.836} & 3.3 & \textbf{0.976} & 5.1 &  & \textbf{0.718} & 2.3 & \textbf{0.944} & 4.7 \\
 & RAPS & \textbf{0.924} & 4.2 & \textbf{0.990} & 5.2 &  & \textbf{0.855} & 3.4 & \textbf{0.978} & 5.0 &  & \textbf{0.704} & 2.2 & \textbf{0.962} & 4.8 \\
\bottomrule
\end{tabular*}
\end{table*}

\subsection{Parameter-free and learning-based weights}
\label{ssec:weights}
Table~\ref{tab:learned} compares the two weight constructions in
Figure~\ref{fig:pipeline}. One calibration half fits $w_\psi$; both variants
then estimate $\hat q$ from the other half. This matched comparison isolates
the effect of the weight. Table~\ref{tab:encp} uses all calibration episodes
for the parameter-free variant because it requires no fitting. Neither
procedure uses val-unseen data for fitting or calibration.

Based on Table~\ref{tab:learned}, at $\alpha=0.10$, step coverage differs by at most $0.006$ across the four R2R
VLNs. The learned variant returns larger sets on all four, with the largest
increase on DUET, from $6.83$ to $7.63$ actions. This added set size does not
produce a consistent coverage gain: coverage rises on DUET and R-osc but falls
on HAMT and R-prev. The learned target identifies confidently wrong steps, but
learning changes where ENCP enlarges a set. We use the parameter-free variant for the main results because it
gives comparable coverage with smaller sets and uses every calibration episode.

\begin{table}[t]
\centering
\caption{The learned weight enlarges sets without a consistent step-coverage
gain. Results use R2R val-unseen, THR, and $\alpha=0.10$.}
\label{tab:learned}
\setlength{\tabcolsep}{5pt}\footnotesize
\begin{tabular}{lcccc}
\toprule
& \multicolumn{2}{c}{Parameter-free} & \multicolumn{2}{c}{Learning-based} \\
\cmidrule(lr){2-3}\cmidrule(lr){4-5}
\textbf{VLN} & $\covstep$ & $\overline{|C|}$ & $\covstep$ & $\overline{|C|}$ \\
\midrule
DUET   & 0.968 & 6.83 & 0.974 & 7.63 \\
HAMT   & 0.952 & 3.95 & 0.949 & 4.01 \\
R-prev & 0.973 & 4.56 & 0.971 & 4.60 \\
R-osc  & 0.970 & 4.49 & 0.971 & 4.63 \\
\bottomrule
\end{tabular}
\end{table}

\subsection{Simulated ENCP help-seeking (DUET, R2R)}
\label{ssec:closedloop}
We test one use of ENCP at deployment: the agent acts when
$|\Cdep|\leq\tau$ and asks for help when $|\Cdep|>\tau$. A larger set therefore
causes the agent to defer.

This experiment simulates help-seeking without human operators. When ENCP requests help because $|\Cdep|>\tau$, the agent executes the ground-truth action supplied by the simulator instead of its policy-selected action. The simulator has access to the goal and complete navigation graph and is assumed to make no errors.

Without help, DUET succeeds on $71.2\%$ of episodes. Querying on sets larger
than 15 actions raises success to $76.9\%$ at a $6.4\%$ ask rate. Lowering
$\tau$ to 8 raises success to $91.1\%$ at a $32.3\%$ ask rate, and querying on
every non-singleton set reaches $98.8\%$ at a $68.7\%$ ask rate
(Figure~\ref{fig:closedloop}). Thus, ENCP set size can control how often the
agent defers to assistance.

This experiment does not select a deployment threshold. We sweep $\tau$ after
observing val-unseen, and the same $\tau$ produces different ask rates for
policies with different set-size distributions (Table~\ref{tab:encp}). The
intervention also changes later observations, so the static ENCP coverage
result does not certify the assisted trajectories.

\begin{figure}[t]
\centering
\includegraphics[width=\columnwidth]{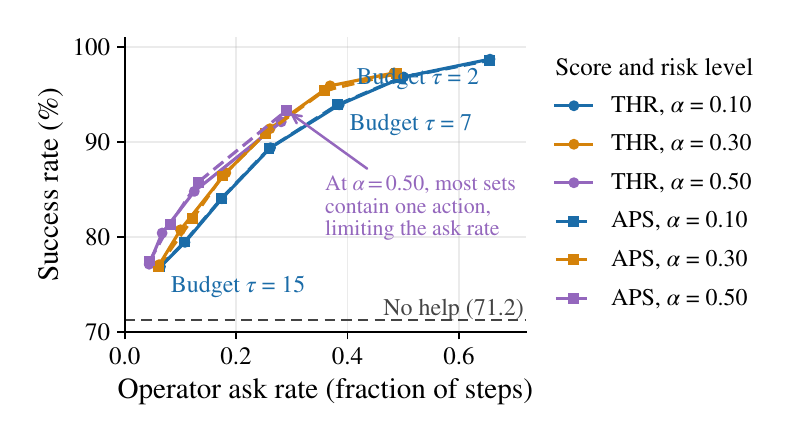}
\caption{Simulation of ENCP help-seeking: when the prediction set exceeds the size budget, the simulator supplies the ground-truth action instead of the policy-selected action.}
\label{fig:closedloop}
\end{figure}

\section{Discussion}
\label{sec:disc}

\textbf{What a large prediction set means.}
Table~\ref{tab:encp} reports both coverage and mean set size because coverage
alone does not measure how often ENCP leaves many actions unresolved. For
parameter-free THR, every admissible action enters the set exactly when
\begin{equation}
\hat q(\alpha)\geq
\frac{1-\min_{a\in\mathcal A_t}p(a\mid x_t)}
{2-\pmax(x_t)}.
\label{eq:thr-saturation}
\end{equation}
The condition depends on the least probable admissible action, as well as the
calibrated threshold and maximum policy probability. A full set signals that
the size rule cannot select one action. It does not rank the candidates, so the
policy still supplies their order. The different mean sizes $\overline{|C|}$ in
Table~\ref{tab:encp} also show why one numerical size budget should not be
assumed to have the same meaning across policies and datasets.

\textbf{Parameter-free and learned weights.}
Table~\ref{tab:learned} shows similar step coverage for the two weight
constructions. In every displayed row, the parameter-free rule has the smaller
mean set and requires no fitted weight model. This evidence supports using the
parameter-free rule as the main ENCP variant in this paper. Table~\ref{tab:learned} compares one feature set, training target, and calibration
split, so it does not rule out a more efficient learned rule under another
design.

\textbf{Coverage and efficiency.}
Figure~\ref{fig:theorem-check} evaluates the stricter event that every ground-truth
action in an episode is covered, whereas Table~\ref{tab:encp} reports the
average fraction of covered steps and the corresponding mean set size. HAMT--REVERIE in Figure~\ref{fig:theorem-check} is the only condition above the trajectory target when calibrating on validation-seen episodes and evaluating on validation-unseen episodes, although this point estimate carries no coverage guarantee because these episode sets are not exchangeable. Taking episode maximum protects the stricter event, but one difficult step can set the calibration score for the whole episode. The set sizes in
Table~\ref{tab:encp} are therefore part of the result rather than a secondary
diagnostic.

\textbf{Closed-loop interpretation.}
The THR and APS curves nearly coincide within each risk level, which indicates
that the set-size trigger is not sensitive to the choice between these two
scores in this simulation. The risk level has a larger operational effect.
Raising $\alpha$ makes more sets singletons, which can improve the success rate
at a comparable ask rate but also lowers the maximum attainable ask rate. Thus,
$\tau$ cannot be read as a fixed operator budget, even for one policy.

\textbf{Limitations.}
The closed-loop study uses a simulated ground-truth assistant with access to the goal and complete navigation graph; no human operators were evaluated. Because the set-size threshold $\tau$ was swept on val-unseen, the resulting curve does not measure performance at a held-out deployment setting. Moreover, the same value of $\tau$ may produce different ask rates across VLN policies. ENCP also assumes a finite action space. Extending ENCP to continuous VLN~\cite{krantz2020vlnce} would require conformal scores for waypoints or continuous controls, together with a separate evaluation of human assistance.

\section{Conclusion}
\label{sec:conclusion}

Step-pooled conformal prediction ignores within-trajectory dependence in VLN and can collapse to the argmax policy. ENCP instead calibrates one maximum normalized score per episode, providing route-level coverage of the teacher action with probability at least \(1-\alpha\) under exchangeable episodes. Across R2R and REVERIE, ENCP responds as expected to \(\alpha\) and meets the reported step-coverage targets. Learned weighting redistributes uncertainty but yields larger prediction sets, while parameter-free weighting provides a better coverage–set-size trade-off. Set size also serves as a practical help-seeking signal, with tighter budgets increasing oracle queries and improving success. For deployment, size budgets should be selected on held-out data, recalibrated under distribution shift, and validated with human operators.

\bibliographystyle{IEEEtran}
\bibliography{references}

@inproceedings{anderson2018r2r, author = {Anderson, Peter and Wu, Qi and Teney, Damien and Bruce, Jake and Johnson, Mark and S\"underhauf, Niko and Reid, Ian and Gould, Stephen and van den Hengel, Anton}, title = {Vision-and-Language Navigation: Interpreting Visually-Grounded Navigation Instructions in Real Environments}, booktitle = {Proc. CVPR}, pages = {3674--3683}, year = {2018}, doi = {10.1109/CVPR.2018.00387}, note = {arXiv:1711.07280}}

@inproceedings{chang2017matterport, author = {Chang, Angel and Dai, Angela and Funkhouser, Thomas and Halber, Maciej and Nie{\ss}ner, Matthias and Savva, Manolis and Song, Shuran and Zeng, Andy and Zhang, Yinda}, title = {{Matterport3D}: Learning from {RGB-D} Data in Indoor Environments}, booktitle = {Proc. 3DV}, pages = {667--676}, year = {2017}, doi = {10.1109/3DV.2017.00081}, note = {arXiv:1709.06158}}

@inproceedings{ma2019selfmonitoring, author = {Ma, Chih-Yao and Lu, Jiasen and Wu, Zuxuan and AlRegib, Ghassan and Kira, Zsolt and Socher, Richard and Xiong, Caiming}, title = {Self-Monitoring Navigation Agent via Auxiliary Progress Estimation}, booktitle = {Proc. ICLR}, year = {2019}, note = {arXiv:1901.03035}}

@inproceedings{tan2019envdrop, author = {Tan, Hao and Yu, Licheng and Bansal, Mohit}, title = {Learning to Navigate Unseen Environments: Back Translation with Environmental Dropout}, booktitle = {Proc. NAACL}, pages = {2610--2621}, year = {2019}, doi = {10.18653/v1/N19-1268}, note = {arXiv:1904.04195}}

@inproceedings{hong2021vlnbert, author = {Hong, Yicong and Wu, Qi and Qi, Yuankai and Rodriguez-Opazo, Cristian and Gould, Stephen}, title = {{VLN BERT}: A Recurrent Vision-and-Language {BERT} for Navigation}, booktitle = {Proc. CVPR}, pages = {1643--1653}, year = {2021}, doi = {10.1109/CVPR46437.2021.00169}}

@inproceedings{chen2021hamt, author = {Chen, Shizhe and Guhur, Pierre-Louis and Schmid, Cordelia and Laptev, Ivan}, title = {History Aware Multimodal Transformer for Vision-and-Language Navigation}, booktitle = {Proc. NeurIPS}, pages = {5834--5847}, year = {2021}, note = {arXiv:2110.13309}}

@inproceedings{chen2022duet, author = {Chen, Shizhe and Guhur, Pierre-Louis and Tapaswi, Makarand and Schmid, Cordelia and Laptev, Ivan}, title = {Think Global, Act Local: Dual-Scale Graph Transformer for Vision-and-Language Navigation}, booktitle = {Proc. CVPR}, pages = {16537--16547}, year = {2022}, note = {arXiv:2202.11742}}

@inproceedings{krantz2020vlnce, author = {Krantz, Jacob and Wijmans, Erik and Majumdar, Arjun and Batra, Dhruv and Lee, Stefan}, title = {Beyond the Nav-Graph: Vision-and-Language Navigation in Continuous Environments}, booktitle = {Proc. ECCV}, pages = {104--120}, year = {2020}, note = {arXiv:2004.02857}}

@article{zhang2024survey, author = {Zhang, Yue and Ma, Ziqiao and Li, Jialu and Qiao, Yanyuan and Wang, Zun and Chai, Joyce and Wu, Qi and Bansal, Mohit and Kordjamshidi, Parisa}, title = {Vision-and-Language Navigation Today and Tomorrow: A Survey in the Era of Foundation Models}, journal = {Trans. Mach. Learn. Res.}, year = {2024}, note = {arXiv:2407.07035}}

@inproceedings{ross2011reduction, author = {Ross, St\'ephane and Gordon, Geoffrey and Bagnell, Drew}, title = {A Reduction of Imitation Learning and Structured Prediction to No-Regret Online Learning}, booktitle = {Proc. AISTATS}, volume = {15}, pages = {627--635}, year = {2011}}

@article{amodei2016concrete, author = {Amodei, Dario and Olah, Chris and Steinhardt, Jacob and Christiano, Paul and Schulman, John and Man\'e, Dan}, title = {Concrete Problems in {AI} Safety}, journal = {arXiv:1606.06565}, year = {2016}}

@article{hendrycks2021unsolved, author = {Hendrycks, Dan and Carlini, Nicholas and Schulman, John and Steinhardt, Jacob}, title = {Unsolved Problems in {ML} Safety}, journal = {arXiv:2109.13916}, year = {2021}}

@article{lee2004trust, author = {Lee, John D. and See, Katrina A.}, title = {Trust in Automation: Designing for Appropriate Reliance}, journal = {Human Factors}, volume = {46}, number = {1}, pages = {50--80}, year = {2004}, doi = {10.1518/hfes.46.1.50.30392}}

@article{parasuraman1997humans, author = {Parasuraman, Raja and Riley, Victor}, title = {Humans and Automation: Use, Misuse, Disuse, Abuse}, journal = {Human Factors}, volume = {39}, number = {2}, pages = {230--253}, year = {1997}, doi = {10.1518/001872097778543886}}

@inproceedings{ovadia2019trust, author = {Ovadia, Yaniv and Fertig, Emily and Ren, Jie and Nado, Zachary and Sculley, D. and Nowozin, Sebastian and Dillon, Joshua V. and Lakshminarayanan, Balaji and Snoek, Jasper}, title = {Can You Trust Your Model's Uncertainty? Evaluating Predictive Uncertainty Under Dataset Shift}, booktitle = {Proc. NeurIPS}, year = {2019}, note = {arXiv:1906.02530}}

@inproceedings{guo2017calibration, author = {Guo, Chuan and Pleiss, Geoff and Sun, Yu and Weinberger, Kilian Q.}, title = {On Calibration of Modern Neural Networks}, booktitle = {Proc. ICML}, volume = {70}, pages = {1321--1330}, year = {2017}, note = {arXiv:1706.04599}}

@inproceedings{gal2016dropout, author = {Gal, Yarin and Ghahramani, Zoubin}, title = {Dropout as a Bayesian Approximation: Representing Model Uncertainty in Deep Learning}, booktitle = {Proc. ICML}, volume = {48}, pages = {1050--1059}, year = {2016}, note = {arXiv:1506.02142}}

@inproceedings{lakshminarayanan2017ensembles, author = {Lakshminarayanan, Balaji and Pritzel, Alexander and Blundell, Charles}, title = {Simple and Scalable Predictive Uncertainty Estimation using Deep Ensembles}, booktitle = {Proc. NeurIPS}, year = {2017}, note = {arXiv:1612.01474}}

@inproceedings{geifman2017selective, author = {Geifman, Yonatan and El-Yaniv, Ran}, title = {Selective Classification for Deep Neural Networks}, booktitle = {Proc. NeurIPS}, year = {2017}, note = {arXiv:1705.08500}}

@book{thrun2005probabilistic, author = {Thrun, Sebastian and Burgard, Wolfram and Fox, Dieter}, title = {Probabilistic Robotics}, publisher = {MIT Press}, year = {2005}}

@book{vovk2005alrw, author = {Vovk, Vladimir and Gammerman, Alexander and Shafer, Glenn}, title = {Algorithmic Learning in a Random World}, publisher = {Springer}, year = {2005}}

@article{lei2018distributionfree, author = {Lei, Jing and G'Sell, Max and Rinaldo, Alessandro and Tibshirani, Ryan J. and Wasserman, Larry}, title = {Distribution-Free Predictive Inference for Regression}, journal = {J. Amer. Statist. Assoc.}, volume = {113}, number = {523}, pages = {1094--1111}, year = {2018}, doi = {10.1080/01621459.2017.1307116}}

@article{angelopoulos2023gentle, author = {Angelopoulos, Anastasios N. and Bates, Stephen}, title = {Conformal Prediction: A Gentle Introduction}, journal = {Found. Trends Mach. Learn.}, volume = {16}, number = {4}, pages = {494--591}, year = {2023}, doi = {10.1561/2200000101}}

@article{sadinle2019least, author = {Sadinle, Mauricio and Lei, Jing and Wasserman, Larry}, title = {Least Ambiguous Set-Valued Classifiers with Bounded Error Levels}, journal = {J. Amer. Statist. Assoc.}, volume = {114}, number = {525}, pages = {223--234}, year = {2019}, doi = {10.1080/01621459.2017.1395341}}

@inproceedings{romano2020aps, author = {Romano, Yaniv and Sesia, Matteo and Cand\`es, Emmanuel J.}, title = {Classification with Valid and Adaptive Coverage}, booktitle = {Proc. NeurIPS}, year = {2020}, note = {arXiv:2006.02544}}

@inproceedings{angelopoulos2021raps, author = {Angelopoulos, Anastasios N. and Bates, Stephen and Malik, Jitendra and Jordan, Michael I.}, title = {Uncertainty Sets for Image Classifiers Using Conformal Prediction}, booktitle = {Proc. ICLR}, year = {2021}, note = {arXiv:2009.14193}}

@inproceedings{tibshirani2019covariateshift, author = {Tibshirani, Ryan J. and Barber, Rina Foygel and Cand\`es, Emmanuel J. and Ramdas, Aaditya}, title = {Conformal Prediction Under Covariate Shift}, booktitle = {Proc. NeurIPS}, year = {2019}, note = {arXiv:1904.06019}}

@article{barber2023beyond, author = {Barber, Rina Foygel and Cand\`es, Emmanuel J. and Ramdas, Aaditya and Tibshirani, Ryan J.}, title = {Conformal Prediction Beyond Exchangeability}, journal = {Ann. Statist.}, volume = {51}, number = {2}, pages = {816--845}, year = {2023}, doi = {10.1214/23-AOS2276}}

@inproceedings{ren2023knowno, author = {Ren, Allen Z. and Dixit, Anushri and Bodrova, Alexandra and Singh, Sumeet and Tu, Stephen and Brown, Noah and Xu, Peng and Takayama, Leila and Xia, Fei and Varley, Jake and Xu, Zhenjia and Sadigh, Dorsa and Zeng, Andy and Majumdar, Anirudha}, title = {Robots That Ask For Help: Uncertainty Alignment for Large Language Model Planners}, booktitle = {Proc. CoRL}, volume = {229}, pages = {661--682}, year = {2023}, note = {arXiv:2307.01928}}

@article{lindemann2023conformal, author = {Lindemann, Lars and Cleaveland, Matthew and Shim, Gihyun and Pappas, George J.}, title = {Safe Planning in Dynamic Environments Using Conformal Prediction}, journal = {IEEE Robot. Autom. Lett.}, volume = {8}, number = {8}, pages = {5116--5123}, year = {2023}, doi = {10.1109/LRA.2023.3292071}, note = {arXiv:2210.10254}}

@inproceedings{dixit2023multiagent, author = {Dixit, Anushri and Lindemann, Lars and Wei, Skylar X. and Cleaveland, Matthew and Pappas, George J. and Burdick, Joel W.}, title = {Adaptive Conformal Prediction for Motion Planning among Dynamic Agents}, booktitle = {Proc. L4DC}, volume = {211}, pages = {300--314}, year = {2023}, note = {arXiv:2212.00278}}

@inproceedings{liang2024introspective, author = {Liang, Kaiqu and Zhang, Zixu and Fisac, Jaime Fern\'andez}, title = {Introspective Planning: Aligning Robots' Uncertainty with Inherent Task Ambiguity}, booktitle = {Proc. NeurIPS}, year = {2024}, note = {arXiv:2402.06529}}

@inproceedings{abraham2025ask, author = {Abraham, Savitha Sam and Garg, Sourav and Dayoub, Feras}, title = {To Ask or Not to Ask? Detecting Absence of Information in Vision and Language Navigation}, booktitle = {Proc. WACV}, pages = {7480--7489}, year = {2025}, doi = {10.1109/WACV61041.2025.00727}, note = {arXiv:2411.05831}}

@inproceedings{qi2020reverie, title={{REVERIE}: Remote Embodied Visual Referring Expression in Real Indoor Environments}, author={Qi, Yuankai and Wu, Qi and Anderson, Peter and Wang, Xin and Wang, William Yang and Shen, Chunhua and van den Hengel, Anton}, booktitle={Proc. CVPR}, pages = {9982--9991}, year={2020}, doi = {10.1109/CVPR42600.2020.01000}}

@inproceedings{stankeviciute2021cfrnn, author = {Stankevi{\v c}i{\=u}t{\.e}, Kamil{\.e} and Alaa, Ahmed M. and van der Schaar, Mihaela}, title = {Conformal Time-series Forecasting}, booktitle = {Proc. NeurIPS}, volume = {34}, pages = {6216--6228}, year = {2021}}

\end{document}